\documentclass[journal]{IEEEtran}

\usepackage[utf8]{inputenc}
\usepackage{graphicx}
\usepackage{amsmath}
\usepackage{amssymb}
\usepackage{amsthm}
\usepackage{amsfonts}
\usepackage{acronym}
\usepackage{setspace}
\usepackage[noadjust]{cite}
\usepackage{algorithm,algorithmic}
\newcommand{\figscale}{1.0}

\newcommand{\FIG}[1]{Fig.~\ref{#1}}
\newcommand{\SEC}[1]{Section~\ref{#1}}
\newcommand{\TAB}[1]{Table~\ref{#1}}
\newcommand{\ALG}[1]{Algorithm~\ref{#1}}

\newtheorem{theorem}{Theorem}
\newtheorem{lemma}{Lemma}

\newtheorem{example}{Example}

\acrodef{PFL}{pipelined federated learning}
\acrodef{MLP}{multi-layer perceptron}
\acrodef{CNN}{convolutional neural network}
\acrodef{SGD}{stochastic gradient descent}
\acrodef{LSTM}{long short-term memory}
\acrodef{RNN}{recurrent neural network}
\acrodef{iid}[i.i.d.]{independent identically distributed}
\acrodef{KKT}{Karush-Kuhn-Tucker}

\begin{document}

\title{Clustered Scheduling and Communication Pipelining For Efficient Resource Management Of Wireless Federated Learning}

\author{Cihat~Ke\c{c}eci,
	Mohammad~Shaqfeh,~\IEEEmembership{Member,~IEEE,}
	Fawaz~Al-Qahtani,~\IEEEmembership{Member,~IEEE,}
	Muhammad~Ismail,~\IEEEmembership{Senior~Member,~IEEE,}
	and~Erchin~Serpedin,~\IEEEmembership{Fellow,~IEEE}%
	\thanks{This publication was made possible by NPRP13S-0127-200182 from the Qatar National Research Fund (a member of Qatar Foundation). The statements made herein are solely the responsibility of the authors. (Corresponding author: Cihat Ke\c{c}eci.)}%
	\thanks{Cihat Ke\c{c}eci and Erchin Serpedin are with the Department of Electrical and Computer Engineering, Texas A\&M University, College Station, TX 77843, USA (e-mail: \{kececi,eserpedin\}@tamu.edu).}%
	\thanks{Mohammad Shaqfeh is with the Department of Electrical and Computer Engineering, Texas A\&M University at Qatar, Doha, Qatar (e-mail: mohammad.shaqfeh@qatar.tamu.edu).}%
	\thanks{Fawaz~Al-Qahtani is with Qatar Foundation Research, Development and Innovation RDI, Doha, Qatar (e-mail: fawalqahtani@qf.org.qa).}%
	\thanks{Muhammad Ismail is with the Department of Computer Science, Tennessee Tech University, Cookeville, TN 38505 USA (e-mail: mismail@tntech.edu).}%
}

\maketitle

\begin{abstract}
This paper proposes using communication pipelining to enhance the wireless spectrum utilization efficiency and convergence speed of federated learning in mobile edge computing applications. Due to limited wireless sub-channels, a subset of the total clients is scheduled in each iteration of federated learning algorithms. On the other hand, the scheduled clients wait for the slowest client to finish its computation. We propose to first cluster the clients based on the time they need per iteration to compute the local gradients of the federated learning model. Then, we schedule a mixture of clients from all clusters to send their local updates in a pipelined manner. In this way, instead of just waiting for the slower clients to finish their computation, more clients can participate in each iteration. While the time duration of a single iteration does not change, the proposed method can significantly reduce the number of required iterations to achieve a target accuracy. We provide a generic formulation for optimal client clustering under different settings, and we analytically derive an efficient algorithm for obtaining the optimal solution. We also provide numerical results to demonstrate the gains of the proposed method for different datasets and deep learning architectures.

% In this paper, we propose using communication pipelining to enhance the wireless spectrum utilization efficiency and speed up the convergence of federated learning in mobile edge computing applications. Due to spectrum scarcity, the number of wireless sub-channels is limited with respect to the total number of participating mobile clients. So, in each iteration of the training algorithm, a subset of the total clients is scheduled. Another problem stems from the heterogeneity in the processing and storage capabilities of mobile devices. 
% In each iteration, the clients wait for the slowest client to finish its computation.
% We propose to first cluster the clients based on the time they need per iteration to compute the local gradients of the federated learning model. Then, we schedule a mixture of clients from all clusters such that they send their local updates in a pipelined manner. In this way, instead of just waiting for the slower clients to finish their computation, more clients can participate in each iteration.
% While the time duration of a single iteration does not change, the proposed method can significantly reduce the number of required iterations to achieve a target accuracy due to the involvement of more clients per iteration.
% We provide a generic formulation for the optimal client clustering under different settings and we analytically derive an efficient algorithm for obtaining the optimal solution.
% We also provide numerical results to demonstrate the gains of the proposed method for different datasets and deep learning architectures.
\end{abstract}

\begin{IEEEkeywords}
Federated learning, spectrum efficiency, communication pipelining, clustered scheduling, mobile edge computing
\end{IEEEkeywords}

%---------------------------------------------------------
\section{Introduction}

Federated learning is a promising distributed machine learning approach in which a central server (coordinator) and a number of data-holding clients cooperate in the training of a common machine learning model without exchanging data \cite{main_FL}. Only model parameters are shared and updated using an iterative gradient descent algorithm. Hence, learning from private clients' data becomes possible while preserving the privacy of the clients'. 
Furthermore, federated learning enables achieving mutual benefits for all participating clients since training the algorithm using larger datasets by aggregating the private data of all participants leads to significant enhancements in the performance of data-hungry models. Moreover, federated learning enables better utilization of the distributed computational and storage resources. In particular, federated learning can be applied in mobile edge computing of wireless networks, e.g.~\cite{FL_Survey, Bin-Lataief_Survey}.

Research in federated learning has many aspects. One of them is exploring the variations among distributed gradient-descent training algorithms and analyzing the convergence performance of these algorithms with respect to the number of local and global iterations \cite{Adaptive_FL_JSAC}.
Another attractive research area is managing and optimizing the distributed computing and communication resources \cite{dinh2020federated}. 
In particular, one of the main challenges in federated learning is the heterogeneity across the clients in terms of their computing and data resources. 
In this work, we focus on enhancing the utilization efficiency of wireless spectrum in federated learning. 
Consequently, enhancing the spectrum utilization leads to enhancing the prediction accuracy of the learning model at the end of the training process.
Hence, achieving a target prediction accuracy becomes possible with a fewer number of iterations.

We propose using communication pipelining to enhance the spectrum utilization efficiency and speed up the convergence of federated learning. The primary motivation behind the proposed concept is that the computation time varies a lot between devices due to several reasons, such as different processor speeds or different sizes of the local dataset \cite{yang2019energy}. 
So, in an iterative gradient algorithm, the time of each iteration will be limited by the slowest client. To overcome this bottleneck, we propose that instead of just waiting for the slow clients to finish the computation of their local gradients, we can schedule an extra number of clients within the same iteration without requiring extra spectral bandwidth. A mixture of fast and slow clients should be scheduled in each iteration. The fast clients will use the spectrum first, while the slow clients are still running their computation of the local update. In this way, the number of clients who participate in each iteration can be increased while preserving the total duration of each iteration.
So, the main idea is to cluster clients based on their computation time and then schedule a balanced mixture of clients from different clusters.
The merit of this proposed concept is not necessarily in increasing the number of scheduled clients in each iteration but rather in allocating the wireless resources more efficiently. So, for example, fewer resources can be used while maintaining the number of clients per iteration. In principle, pipelining is an added degree of freedom in the resource allocation problem that can be adequately utilized to adjust the trade-off between allocated resources and the number of clients to be scheduled to achieve the best outcome of the federated learning process.

It is essential to highlight that the proposed concept of \ac{PFL} is a general concept that can be applied to any variation of the gradient descent algorithm, whether batch or stochastic, and for both primal or dual optimization methods. 
In this paper, we just show, as an example, the formulation for a primal method with client sampling in each iteration and one local iteration of full-gradient (batch) computation in each global iteration.
We refer to this method as Clustered-Scheduling Gradient-Descent (CS-GD). The extension of the concept to other cases, such as having more local iterations per client, is straightforward as long as the needed computation time per client is quantified.

\subsection{Related Work}

The application of federated learning is discussed in several studies, such as \cite{main_FL, bonawitz2019towards, li2020federated, zhang2021dual}.
A seminal work on federated learning introducing the federated averaging algorithm is proposed in \cite{main_FL}. Federated averaging algorithm aggregates the updated local weights from each client in each global iteration.  
A comprehensive survey on federated learning is provided in \cite{bonawitz2019towards, li2020federated} by discussing the arising challenges, problems, and possible solutions regarding federated learning.
The idea of federated learning is applied to the traffic prediction in \cite{zhang2021dual}.

In \cite{Wadu}, the authors investigated the problem of client scheduling and resource block (RB) allocation to enhance the performance of model training using FL.
In \cite{Tran}, the authors adopted FL in wireless networks as a resource allocation optimization problem that captures the trade-off between FL convergence wall clock time and energy consumption of UEs with heterogeneous computing and power resources.   
In \cite{Yang}, the authors proposed a novel over-the-air computation based-approach for fast global model aggregation via exploring the superposition property of a wireless multiple-access channel. This is achieved by joint device selection and beamforming design using the difference of convex function representation.
In the same direction, the authors of \cite{Zhu} proposed a novel digital version of broadband over-the-air aggregation, called one-bit broadband digital aggregation (OBDA), featuring one-bit gradient quantization followed by digital quadrature amplitude modulation (QAM) at edge devices and over-the-air majority-voting based decoding at edge server. 

Some other interesting works in the literature have considered the client scheduling problem, e.g.~\cite{Poor_TCOM, Channel_aware_Scheduling_TWC, Bandit_Scheduling, Age-based_Scheduling, W_Shi_TWC, Nishio_icc2019}.
In \cite{Poor_TCOM}, three practical scheduling policies, which are random scheduling, round Robin scheduling, and proportional fair scheduling, were compared. In \cite{Channel_aware_Scheduling_TWC}, a probabilistic scheduling framework that considers both channel quality and local update importance was proposed. In \cite{Bandit_Scheduling}, a multi-armed bandit-based framework was proposed for online client scheduling. In \cite{Age-based_Scheduling}, scheduling based on a metric called ``age of information'' was proposed, and in \cite{W_Shi_TWC}, joint client scheduling and resource allocation was proposed.
The training and wireless transmission parameters are jointly optimized in \cite{chen2020joint}.
The works mentioned above do not take into consideration the variations in computational time among the clients. This is an essential factor to consider when scheduling the clients in order to avoid wasted time waiting for clients that need extended computing times.
This is the main motivation for this work.
The computational heterogeneity of the clients was considered in \cite{Nishio_icc2019}. In this work, we develop the concept further by introducing client clustering and communication pipelining.

\subsection{Contributions}
We have three major contributions in this paper, which are summarized as follows:
\begin{itemize}
    \item We propose a novel federated learning concept \ac{PFL} in which the communication for the model parameter updates is performed in a pipelined manner. The client updates in each round are pipelined by assigning the clients into clusters by their computation time.
    
    \item We investigate the problem of optimal client clustering in a generic form that can be applicable to all particular scenarios. We show how to formulate the problem, and then we present rigorous analytical steps to derive an optimal and efficient algorithm to solve the problem. The solutions steps involve the change of variables step, Lagrangian dual-problem formulation, manipulations of the \ac{KKT} necessary and sufficient conditions for optimality, hypothesis testing approach to find the active inequality constraints, and finally, an elegant geometric interpretation that helps in obtaining an efficient algorithm to solve the optimization problem.

    \item We demonstrate the advantages of the proposed \ac{PFL} using extensive numerical examples for different federated learning tasks. We compare the gain introduced by the proposed pipelining scheme compared to the conventional federated learning scheme using the MNIST \cite{lecun2010mnist}, the federated EMNIST \cite{cohen2017emnist}, and the federated Shakespeare \cite{main_FL,shakespeare2007complete} datasets by employing a variety neural network architectures such as \ac{MLP}, \ac{CNN}, and \ac{RNN}. The simulation results suggest that employing \ac{PFL} decreases the number of communication rounds required to train a deep neural network.
\end{itemize}

The rest of the paper is organized as follows. The novel \ac{PFL} scheme is introduced in \SEC{sec:pfl}. A generic client clustering algorithm is proposed and solved in \SEC{sec:clustering}. Extensive simulations are performed to validate the proposed method in \SEC{sec:simulations}. Finally, the conclusions are drawn in \SEC{sec:conclusions}. A summary of symbols and notations is provided in \TAB{tab:nomenclature}.

\begin{table}[htbp]
    \centering
    \caption{Summary of symbols and notations}
    \label{tab:nomenclature}
    \begin{tabular}{ll}
    \hline
    Notation & Description \\ % & Notation & Description \\
    \hline
    $M$ & Total number of clients \\
    $n$ & Total number of data samples \\
    $n_m$ & Number of data samples for client $m$ \\
    $n_{\min}$ & Min. data samples per client \\
    $n_{\max}$ & Max. data samples per client \\
    $F(\mathbf{w})$ & Global loss function \\
    $F_m(\mathbf{w})$ & Loss function for client $m$ \\
    $l(\cdot)$ & Loss function for a data sample \\
    $w$ & Model weights \\
    $\mathbf{x}_{m}^{i}$ & Input values for the model \\
    $y_{m}^{i}$ & Output values for the model \\
    $T$ & Number of global iterations \\
    $\mathcal{S}[t]$ & Active subset of clients at iteration $t$ \\
    $\mathbf{w}[t]$ & Model weights at iteration $t$ \\
    $\mathbf{g}[t]$ & Average gradient \\
    $\mathbf{g}_m[t]$ & Gradient for client $m$ \\
    $\eta_t$ & Learning rate \\
    $N$ & Number of sub-channels \\
    $\tau_m$ & Computation time for client $m$ \\
    $\tau_\text{com}$ & Communication time \\
    $\tau_\text{server}$ & Server time \\
    $\tau_\text{global}[t]$ & Duration of a global iteration \\
    $\tau_\text{total}$ & Total duration of $T$ iterations \\
    $\epsilon$ & Spectral utilization efficiency \\
    $K$ & Number of clusters \\
    $\mathcal{C}_k$ & $k$th cluster \\
    $\theta_k$ & Computation time for cluster $k$ \\
    $\sigma$ & Computation time per sample \\
    $\Delta_\text{global}$ & Extra overall communication time \\
    $S_m$ & Ordered index of client $m$ \\
    $\mathcal{T}[S_m]$ & Function that maps $S_m$ to $\tau_m$ \\
    $\omega_k$ & Index of max. $\tau_m$ in cluster $k$ \\
    $\pi_k$ & Threshold for cluster $k$ \\
    $M_k$ & Number of clients in cluster $k$ \\
    $\delta_k$ & Relaxed parameter $M_k$ \\
    $|\cdot |$ & Cardinality of a set \\
    $\lfloor \cdot \rfloor$ & Floor function \\
    $[ \cdot ]$ & Nearest integer \\
    $\mathcal{L}$ & Lagrangian \\
    $\lambda_k$, $\nu$ & Lagrangian dual variables \\
    \hline
    \end{tabular}
\end{table}

%--------------------------------------------------------
\section{Pipelined Federated Learning (PFL)}
\label{sec:pfl}

A typical federated learning problem is considered with a total of $M$ clients participating in the optimization algorithm without sharing their own data sets. Each client has $n_m$ data samples, where $m$ is the index of the client. 
The total number of data samples is $n = \sum_{m=1}^{M} n_m$.
The goal is to minimize a loss function defined as:
\begin{equation}
    \label{eq:main_problem}
    \min_{\mathbf{w}\in\mathbb{R}^d} F(\mathbf{w}) \equiv \sum_{m=1}^{M} p_m F_m(\mathbf{w})
    \end{equation}
    where $p_m = \frac{n_m}{n}$, and
    \begin{equation}
    F_m(\mathbf{w}) = \frac{1}{n_m} \sum_{i=1}^{n_m} l(\mathbf{w}, \mathbf{x}_{m}^{i}, y_{m}^{i})
\end{equation}
where $l(\mathbf{w}, \mathbf{x}_{m}^{i}, y_{m}^{i})$ is the loss function for the model and $w$, $\mathbf{x}_{m}^{i}$, and $y_{m}^{i}$ are the model weights, the input values, and the output values, respectively.

The model training mechanism is based on a number $T$ of iterations. In each iteration $t$, the server selects a sub-set of the total clients $\mathcal{S}[t]$ and broadcasts an updated model $\mathbf{w}[t]$ to the selected clients. The selected clients calculate the local gradient using their data samples and send the gradient to the server. 
A local gradient computation, at a given global round $t$, is based on averaging the gradient of all data points in the data set of the client:
\begin{equation}
    \label{eq:local_gradient}
    \mathbf{g}_m[t] = \nabla F_m(\mathbf{w}[t])=\frac{1}{n_m}\sum_{i=1}^{n_m} \nabla l(\mathbf{w}[t],\mathbf{x}_{m}^{i}, y_{m}^{i})
\end{equation}
Then the server averages the gradients 
\begin{equation}
    \label{eq:gradients_averaging}
    \mathbf{g}[t] = \frac{\sum_{m\in\mathcal{S}[t]}n_m \mathbf{g}_m[t]}{\sum_{m\in\mathcal{S}[t]}n_m}
\end{equation}
and updates the model according to:
\begin{equation}
    \label{eq:global_update}
    \mathbf{w}[t+1]=\mathbf{w}[t] - \eta_t \mathbf{g}[t]
\end{equation}
where $\eta_t > 0$ is the learning rate in the $t$-th iteration.

The uplink communication channel is considered to be an OFDM channel with a total of $N$ sub-channels. 
So, at a given time, no more than $N$ users can communicate with the server.
Each sub-channel consists of multiple sub-carriers, and the data packets may cover multiple OFDM frames in order to fit the exchanged gradient data between the clients and the server. The number of gradients may be in the order of a million since the deep learning architectures usually have a large number of trainable variables.
Furthermore, the total number of clients is assumed to be much larger than the number of sub-channels $N \ll M$.
Therefore, the server should randomly select a subset of the clients at each iteration based on the available communication resources (i.e., the number of sub-channels). 

The computation time to obtain the local gradients varies among the clients for several reasons, such as different processor speeds and different computation loads. Note that the computation load to calculate \eqref{eq:local_gradient} depends on the total number of local data samples $n_m$. So, even if two clients have the same processor speed, a larger dataset will need a larger computation time than a client with a smaller dataset. This variation in computation time represents a practical consideration.
On the other hand, since the packet size for communicating a local update to the server is the same across all clients (and it depends on the size of the gradient vector $\mathbf{g}$), the communication time is considered to be the same for all clients.
We know that wireless channels’ quality differs among different users due to small-scale and large-scale fading. So, some studies in the literature have taken the channel quality into consideration in the selection of the clients per iteration \cite{Poor_TCOM, Channel_aware_Scheduling_TWC}. However, in this work, we do not want to compromise machine learning best practices for other advantages related to opportunistic communication in order to enhance transmission data rate. Hence, we design a system that can give all clients an almost equal chance to be selected. In particular, we aim to use the largest possible amount of data in training by involving all clients and learning a model that is not biased towards the data of only a subset of the clients who get selected more frequently than others.
In our communication scheme, we assume that the communication over the sub-channels applies a fixed rate that is independent of the particularly selected user. We assume that the fixed transmission rate can be reliably decoded by all users. This is a common scheme in practical communication systems. The wireless channel data rate efficiency could be enhanced by applying adaptive modulation and coding based on the channel conditions. However, this is out of the scope of this manuscript.
We denote the computation time for client $m$ as $\tau_m$, the communication time of every client as $\tau_\text{com}$, and the combined time for the server to update the global model and broadcast it to the clients alongside information about the scheduled clients for next round as $\tau_\text{server}$. 
In order for the server to update the global model and broadcast the new model $\mathbf{w}[t]$ to the clients, it must wait for all local updates.
So, the bottleneck will be to wait for the slowest client until it finishes its local computation.
Therefore, the duration of one global iteration $t$ is given as:
\begin{equation}
    \label{eq:time_one_iteration}
    \tau_\text{global}[t] = \tau_\text{server}+\max_{m\in\mathcal{S}[t]}\tau_m+\tau_\text{com}
    \end{equation}
    and the total time to complete the whole training process over $T$ iterations is given as:
    \begin{equation}
    \tau_\text{total}=\sum_{t=1}^T \tau_\text{global}[t] = T \left(\tau_\text{server}+\tau_\text{com}\right) +\sum_{t=1}^T  \max_{m\in\mathcal{S}[t]}\tau_m
\end{equation}

One might think of excluding the clients with large values of $\tau_m$ to speed up the training process. However, this is not a good option since those clients are likely to be the clients with the largest datasets, and hence, their local updates have higher weights (i.e., they have higher importance).
Furthermore, it is advantageous to involve more data samples in the training process in order to enhance the overall learning accuracy.
So, client exclusion is not recommended. Instead, client scheduling in each global iteration should be based on random sampling so that each client gets a chance to be involved in each global iteration.
On the other hand, just waiting for the slowest client to finish its computation is a waste of the valuable spectral resources since the utilization efficiency $\epsilon$ of the $N$ uplink sub-channels will be
\begin{equation}
    \epsilon = \frac{T \tau_\text{com}}{\tau_\text{total}}  \approx \frac{\tau_\text{com}}{\tau_\text{com}+\tau_\text{server}+\max_m \tau_m}
\end{equation}

In this work, we propose to enhance the spectrum utilization without the need to exclude any client by scheduling a mixture of slow and fast clients in each iteration and making them use the available channels in a pipelined manner. So, the clients are grouped into a number $K$ of clusters based on their computation speed. More details about how to cluster the clients will be given in the next section. In each iteration, the server randomly samples $N$ clients from each cluster. In this way, the total number of clients per iteration will be $K N$. The proposed \ac{PFL} as illustrated in \FIG{fig:protocol}.
\begin{figure*}
    \centering
    \includegraphics[width=0.75\linewidth]{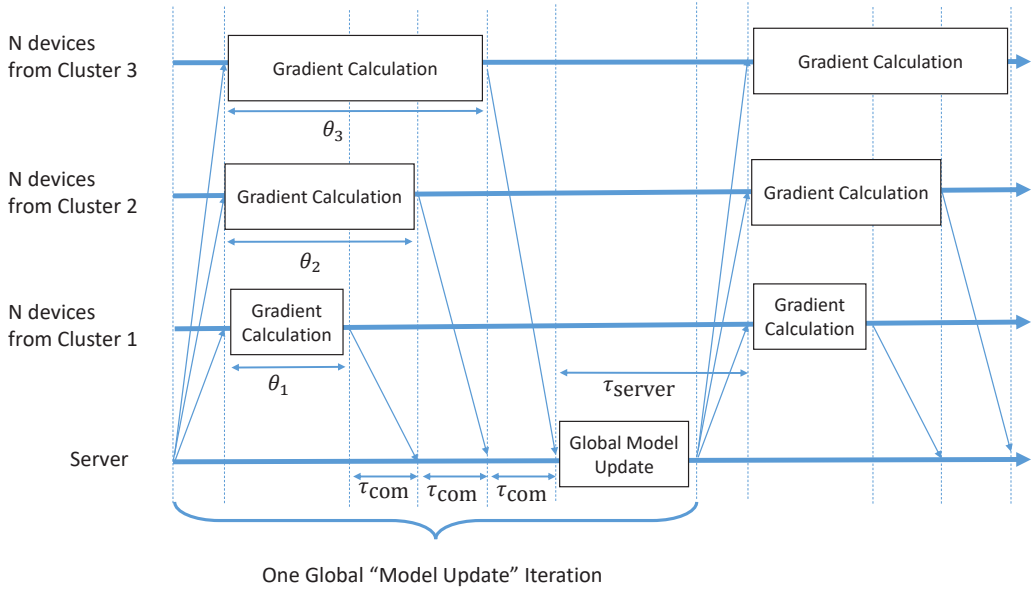}
    \caption{One global update round of CS-GD consisting of pipelined communication rounds. $\theta_k$ where $k=1,2,3$ denotes the computation time for the respective cluster. Each cluster sends the updated parameters in the next communication interval without waiting for the slower clusters.}
    \label{fig:protocol}
\end{figure*}
The utilization efficiency of the wireless sub-channels based on our proposal becomes:
\begin{equation}
    \epsilon_\text{PFL} = \frac{TK \tau_\text{com}}{\tau_\text{total}}  \approx \frac{K\tau_\text{com}}{\tau_\text{com}+\tau_\text{server}+\max_m \tau_m}
    \approx K \epsilon
\end{equation}

Thus, \ac{PFL} does not only produce a significant increase in the utilization efficiency of the spectrum resources, but it does also improve the performance
of the learning algorithm since adding more clients in each iteration will make it converge faster into the globally optimal solution of \eqref{eq:main_problem}.
So, while \ac{PFL} does not decrease the time per iteration \eqref{eq:time_one_iteration}, it can help to achieve a targeted accuracy level for solving \eqref{eq:main_problem}
with less number of iterations $T$. So, it does indeed speed up the convergence performance of the learning algorithm.
This will be demonstrated in the numerical examples section.

%---------------------------------------------------------------------------------------------
\section{Client Clustering}
\label{sec:clustering}

Clients are grouped into $K$ clusters, denoted as $\mathcal{C}_1$, $\mathcal{C}_2$, $\cdots$, $\mathcal{C}_K$. Each client is placed in just one cluster. So, $\sum_{k=1}^K |\mathcal{C}_k|=M$. The computation time for a given cluster $k$ is denoted as $\theta_k$. This represents the time duration given to the selected clients in the cluster
to complete their gradient computations. Then, they synchronously get scheduled to access the communication channels at time $\theta_k$, for a communication duration $\tau_\text{com}$, as illustrated in Fig.~\ref{fig:protocol}.

Clustering the clients has some flexibility and it can be done in several ways. However, there are two strict constraints to be taken into consideration and two favorable, but not strict, requirements. First, all clients in a given cluster must satisfy the constraint that 
\begin{equation}
    \label{eq:constraint1}
    \tau_m \leq \theta_k, \quad \forall  m\in\mathcal{C}_k, \quad k=1,\dots,K.
\end{equation}
Furthermore, we have a constraint for scheduling the clusters into consequent communication frames. Assuming that the clusters are ordered in an ascending way based on their $\theta_k$, the constraint is given as
\begin{equation}
    \label{eq:constraint2}
    \theta_k - \theta_{k-1} \geq \tau_\text{com}, \quad k=2,\dots,K.
\end{equation}
A favorable, but not strict, requirement is to make the sizes of the clusters (i.e., number of clients) almost equal in order to give equal chance for all clients to participate in the training process. Hence, we have
\begin{equation}
    \label{eq:requirement}
    |\mathcal{C}_k| \approx \frac{M}{K}, \quad \forall k=1,\dots,K.
\end{equation}
Another favorable requirement is to satisfy that
\begin{equation}
 \label{eq:requirement2}
    |\mathcal{C}_k| \geq N, \quad \forall k=1,\dots,K,
\end{equation}
in order to fully utilize the wireless resources to the maximum possible efficiency.

The Clustered-Scheduling Gradient-Descent (CS-GD) algorithm is summarized in Algorithm~\ref{alg:csgd}.
This is a generic algorithm that can be applied for any arbitrary client clustering method given that \eqref{eq:constraint1} and \eqref{eq:constraint2} are satisfied. 
Next, we elaborate on the optimization of clustering based on \eqref{eq:requirement} while maintaining the other two constraints in \eqref{eq:constraint1} and \eqref{eq:constraint2}.

\begin{algorithm}[htbp]
\caption{Clustered-Scheduling Gradient-Descent (CS-GD)}
\label{alg:csgd}
\begin{algorithmic}[1]
\renewcommand{\algorithmicrequire}{\textbf{Input:}}
\renewcommand{\algorithmicensure}{\textbf{Output:}}
\REQUIRE The number of sub-channels $N$; the clients' clusters $\mathcal{C}_1$, $\mathcal{C}_2$, $\cdots$, $\mathcal{C}_K$;
the initial model parameters $\mathbf{w}[0]$; the number of global iterations $T$; the learning rate per iteration $\eta_t$ for $t=1,2,\cdots,T$;
the number of data samples per client $n_m, \forall m$.
\ENSURE  The final model $\mathbf{w}[T]$
%\\ \textit{Initialisation} :
%\STATE first statement
%\\ \textit{LOOP Process}
\FOR {global round $t = 1$ to $T$}
    \FOR {cluster $k = 1$ to $K$}
        \STATE Uniformly sample $\mathcal{C}_k$ to obtain a set of $N$ clients, denoted by $\mathcal{S}_k[t]$
        \STATE $\mathcal{S}[t] = \mathcal{S}[t] \cup \mathcal{S}_k[t]$
        \FOR {each client $m$ in $\mathcal{S}_k[t]$}
            \STATE Compute the gradient $\mathbf{g}_m[t]$ according to \eqref{eq:local_gradient}
            \STATE Send gradient $\mathbf{g}_m[t]$ to the server according to the pipeline timing $\theta_k$ of its cluster
        \ENDFOR
    \ENDFOR
    % \STATE $\mathcal{S}[t] = \bigcup_{k=1}^K \mathcal{S}_k[t]$
    \STATE The server aggregates the gradients according to \eqref{eq:gradients_averaging}, and updates the model according to \eqref{eq:global_update}
    %  \IF {($i \ne 0$)}
    % \STATE statement..
    % \ENDIF
\ENDFOR
%\RETURN $P$ 
\end{algorithmic} 
\end{algorithm}

%---------------------------------------------------------------------------------------------
\subsection{Generic Optimization Problem Formulation}

First, we sort the clients in ascending order of their computation time $\tau_m$, and assign the integer index $S_m$ for the order of client $m$.
Hence, if $S_i>S_j$, then $\tau_i \geq \tau_j, \forall S_i, S_j \in \lbrace 1,2,\dots,M \rbrace$. Define the discrete function $\mathcal{T}$ for the ordered computation times of the clients as
\begin{equation}
    \mathcal{T}[S_m] = \tau_m, \quad S_m = 1, 2, \dots, M.
\end{equation}
By definition, $\mathcal{T}[S_m]$ is monotonically non-decreasing.
Here, we would like to provide a generic optimization framework for client clustering that is applicable for any arbitrary monotonically non-decreasing $\mathcal{T}$.
Extension into some special cases will be straightforward.
An example of special cases is the case when the computation time is uniformly distributed among the clients between $\tau_\text{min}$ and  $\tau_\text{max}$,
which results in
\begin{equation}
    \mathcal{T}[i]=\tau_\text{min} + \frac{\tau_\text{max}-\tau_\text{min}}{M-1}(i-1), \quad i \in \lbrace 1,2,\dots,M \rbrace.
\end{equation}
Next, we use the notation $\omega_k$ to denote the index $S_m$ of the client that has the largest computation time among all clients in cluster $\mathcal{C}_k$. Hence, we have
\begin{equation}
    \omega_k = \max_{m \in \mathcal{C}_k} S_m, \quad S_m \in \lbrace 1,2,\dots,M \rbrace.
\end{equation}
Based on the introduced notations for sorting the clients by their computation time, clustering the clients will be based on the following rule:
\begin{equation}
    \label{eq:cluster_set}
    \mathcal{C}_k = \{m~:~\omega_{k-1} < S_m \leq \omega_k\}, \; \forall k\in \lbrace 1,2,\dots,K \rbrace,
\end{equation}
where $\omega_{K}=M$, and $\omega_{0}=0$.
Based on \eqref{eq:cluster_set}, it is straightforward to show that
% \begin{equation}
$|C_k| = \omega_k - \omega_{k-1}$.
% \end{equation}
Now, we can formulate the optimization problem to reduce the variance among the number of clients per cluster since it is favorable to have clusters with almost equal sizes, according to \eqref{eq:requirement}
\begin{subequations}
\label{eq:main_optimization}
\begin{alignat}{3}
    \min_{\omega_1,\dots,\omega_{K-1}} & \quad && \sum_{k=1}^{K} \left( \left( \omega_{k}-\omega_{k-1} \right) - \frac{M}{K} \right)^2  \label{eq:objective_main}\\
    \text{subject to} & && \mathcal{T}[\omega_k] \leq \theta_k,  \quad \forall k\in \lbrace 1,2,\dots,K-1 \rbrace, \label{eq:threshold_main} \\
    & && \omega_{K}=M, \label{eq:all_clients}
\end{alignat}
\end{subequations}
where the constraint \eqref{eq:threshold_main} is based on \eqref{eq:constraint1}. We do not need to include $\omega_K$ as an optimization variable, but rather we added it as
a constraint in \eqref{eq:all_clients} to enforce the optimal solution to include all clients in the clustering method and do not exclude any client.

Before we discuss the solution steps of \eqref{eq:main_optimization}, we will comment first on the selection of constrained thresholds $\theta_k$ in \eqref{eq:threshold_main}.
First, introduce an extra hyper-parameter in the problem formulation to allow for some extra overall time that can be added at each global iteration. The added extra time is denoted by $\Delta_\text{global}$, where $\Delta_\text{global} \geq 0$.
The purpose of this added hyper-parameter is to enable more flexibility in adjusting the number of clients per cluster and consequently to enable reducing the variations among the clusters' sizes.
Furthermore, the extra time $\Delta_\text{global}$ can also be useful in enabling having a larger number of clusters $K$ if needed.
However, there is a trade-off here since adding more time per iteration will reduce the spectrum utilization efficiency to become
 \begin{equation}
    \epsilon = \frac{K \tau_\text{com}}{\tau_\text{com} + \tau_\text{server} + \max\tau_{m} + \Delta_\text{global}}
\end{equation}

We consider $\Delta_\text{global}$ and the total number of clusters $K$ to be hyper-parameters that are set in advance before solving \eqref{eq:main_optimization}.
The selection of $\Delta_\text{global}$ will affect the values of the thresholds $\theta_k$ in the optimization problem constraint \eqref{eq:threshold_main}. 
We can start by examining $\theta_K$ of the last cluster, i.e., the one with largest computation time, which can be adjusted as
% \begin{equation}
$\theta_K=\tau_\text{max}+\Delta_\text{global}$,
% \end{equation}
and then we apply \eqref{eq:constraint2} with equality to adjust $\theta_k$
for clusters $K-1$, $K-2$, $\cdots$, $1$ sequentially. As a result, we will have
\begin{equation}
    \label{eq:setting_theta}
    \theta_k = \tau_\text{max} + \Delta_\text{global} - (K - k) \tau_\text{com}, \quad k=K,\dots,1.
\end{equation}
It should be noted that using the procedure mentioned above to adjust the values of the thresholds $\theta_k$ in \eqref{eq:setting_theta} aims at setting these thresholds
at the maximum possible values in order to enlarge the space of feasible solutions of the main optimization problem \eqref{eq:main_optimization}, and hence, to enable achieving a lower
value of the objective function \eqref{eq:objective_main} at the optimal solution of the problem.  

The upper bound of the possible number of clusters $K$ can be obtained by setting $\theta_1 \geq  \tau_\text{min}$, which can be simplified as $ \tau_\text{max} + \Delta_\text{global} - (K - 1) \tau_\text{com} \geq \tau_\text{min}$.
This inequality is equivalent to
\begin{equation*}
K \leq \frac{\tau_\text{max}-\tau_\text{min}+\Delta_\text{global}}{\tau_\text{com}}+1.
\end{equation*}
Furthermore, since the number of clusters is an integer, we can write
\begin{equation}
    \label{eq:techincal_K}
    K \leq \Big\lfloor \frac{\tau_\text{max}-\tau_\text{min}+\tau_\text{com}+\Delta_\text{global}}{\tau_\text{com}} \Big\rfloor 
\end{equation}
where $\lfloor x \rfloor$ is the floor (rounding to highest lower integer) of $x$.

The above upper bound of $K$ is theoretically correct. However, it is not plausible in practice since it is based on assuming $\theta_1 \geq  \tau_\text{min}$ is satisfied at
strict equality. This means that the first cluster will have just a single client that has the lowest computation time among all clients. This violates \eqref{eq:requirement}. 
Therefore, we alternatively prefer, from practical perspectives, to adjust $K$ by reducing the upper bound by one, which gives  
\begin{equation}
    \label{eq:choosing_K}
    K = \Big\lfloor \frac{\tau_\text{max}-\tau_\text{min}+\Delta_\text{global}}{\tau_\text{com}} \Big\rfloor.
\end{equation}

%---------------------------------------------------------------------------------------------
\subsection{Optimal Solution}
\label{sec:clustering-solution}

It is not straightforward to solve the primal optimization problem in \eqref{eq:main_optimization} analytically since the objective function will
have non-linear terms after expanding the square in \eqref{eq:objective_main}. Therefore, we propose here a change of variable step to transform the problem
formulation into a convex optimization problem. We use the number of clients per cluster $|\mathcal{C}_k|$, denoted as $M_k$, where $k = 1, 2, \dots, K$, as the optimization variables
instead of $\omega_k$. Thus, $M_k = \omega_k - \omega_{k-1}$.
Since the optimization variables $M_k$ are integers, we relax the problem into the continuous domain by replacing $M_k$ by $\delta_k$, which is a real number. The resulting reformulated optimization problem will be
\begin{subequations}
\label{eq:main_optimization3}
\begin{alignat}{3}
    \min_{\delta_{1},\dots,\delta_{K}} & \quad && \sum_{k=1}^{K} \left( \delta_k - \frac{M}{K} \right)^2  \label{eq:objective_main3}\\
    \text{subject to} & && \sum_{i=1}^{k} \delta_i \leq \pi_k,  \quad \forall k\in \lbrace 1, 2, \dots, K-1 \rbrace, \label{eq:threshold_main3} \\
    & && \sum_{i=1}^{K} \delta_i =M. \label{eq:all_clients3}
\end{alignat}
\end{subequations}
% We can then reformulate \eqref{eq:main_optimization} as
% \begin{subequations}
% \label{eq:main_optimization2}
% \begin{alignat}{3}
%     \min_{M_{1},\dots,M_{K}} & \quad && \sum_{k=1}^{K} \left( M_k - \frac{M}{K} \right)^2  \label{eq:objective_main2}\\
%     \text{subject to} & && \sum_{i=1}^{k} M_i \leq \pi_k,  \quad \forall k\in \lbrace 1, 2, \dots, K-1 \rbrace, \label{eq:threshold_main2} \\
%     & && \sum_{i=1}^{K} M_i =M, \label{eq:all_clients2}
% \end{alignat}
% \end{subequations}
where the thresholds $\pi_k$ satisfy 
$\pi_k = |\{m : \tau_m \leq \theta_k\}|$,
which can be equivalently written as
\begin{equation}
\label{eq:pi}
    \pi_k = \max S_m \quad \text{subject to } \mathcal{T}[S_m] \leq \theta_k.
\end{equation}
By solving \eqref{eq:main_optimization3} for $M_k$, $k = 1, \dots, K$, we can then easily obtain $\omega_{k}=\sum_{i=1}^{k} M_i$, for $k = 1, \dots, K$. Then, we can apply clustering using \eqref{eq:cluster_set}.
% Since the optimization variables $M_k$ in \eqref{eq:main_optimization2} are integers, we relax the problem into the continuous domain by replacing $M_k$ by $\delta_k$, which is a real number. The resulting reformulated optimization problem will be
% \begin{subequations}
% \label{eq:main_optimization3}
% \begin{alignat}{3}
%     \min_{\delta_{1},\dots,\delta_{K}} & \quad && \sum_{k=1}^{K} \left( \delta_k - \frac{M}{K} \right)^2  \label{eq:objective_main3}\\
%     \text{subject to} & && \sum_{i=1}^{k} \delta_i \leq \pi_k,  \quad \forall k\in \lbrace 1, 2, \dots, K-1 \rbrace, \label{eq:threshold_main3} \\
%     & && \sum_{i=1}^{K} \delta_i =M. \label{eq:all_clients3}
% \end{alignat}
% \end{subequations}
Problem \eqref{eq:main_optimization3} is a convex optimization problem and it can be solved analytically. After obtaining $\delta_k$, $k = 1, \dots, K$, we can obtain $M_k$
for $k = 1, \dots, K$ by rounding to the nearest integer.
We can also directly obtain $\omega_k$ ($k=1,\dots,K-1$) of the primal optimization problem \eqref{eq:main_optimization} by using
\begin{equation*}
% \label{eq:from_delta_to_omega}
\omega_k = \left[ \sum_{i=1}^{k} \delta_i \right], \quad
k=1,2,\dots,K-1,
\end{equation*}
where $[ x ]$ is the rounding operator to the nearest integer of $x$. It should be noted that rounding to lowest higher integer should not cause any concern regarding 
violating the inequality constraints in \eqref{eq:threshold_main3} or equivalently in \eqref{eq:threshold_main}, since the thresholds $\pi_k$ in \eqref{eq:threshold_main3} are
by definition integers. So, if $\sum_{i=1}^{k} \delta_i$ is not an integer, rounding this summation to the lowest higher integer will also satisfy the constraint.
When we obtain a relaxed (continuous-valued) number of clients per cluster, it is clear that the optimal value in the discrete domain should be by rounding the relaxed number of clients to the nearest integer. Any other value will give a higher variance. Furthermore, rounding to the lower or upper integer will give equal value. So, for example, if we have 7 users and want to distribute them into two groups. The optimal relaxed value will be 3.5 per group. If we round this value, we could have the first group has 4 and the second group has 3, or the first group has 3 and the second group has 4. Both solutions have exactly the same variance. So, in principle, rounding the nearest integer will always be optimal, although it may not be the unique optimal solution.
The Lagrangian dual problem is given by
\begin{subequations}
\label{eq:dual_optimization}
\begin{alignat}{2}
    \min \quad & \mathcal{L}(\delta_{1},\dots,\delta_{K},\lambda_1,\dots,\lambda_{K-1},\nu) \label{eq:objective_dual}\\
    \text{subject to}  \quad & \lambda_i \geq 0,  \quad \forall k\in \lbrace 1, 2, \dots, K-1 \rbrace, \label{eq:threshold_dual} 
\end{alignat}
\end{subequations}
where the Lagrangian $\mathcal{L}$ in \eqref{eq:objective_dual} can be written as
\begin{equation}
\label{eq:Lagrangian}
\begin{split}
\mathcal{L}= & \sum_{k=1}^{K} \left(\delta_k - \frac{M}{K}\right)^2 
+ \sum_{k=1}^{K-1} \lambda_k \left(\sum_{i=1}^{k} \delta_i - \pi_k\right) \\ 
& - \nu \left(\sum_{i=1}^{K} \delta_i - M\right),
\end{split}
\end{equation}
% \begin{equation}
% \label{eq:Lagrangian}
% \mathcal{L} = \sum_{k=1}^{K} \left(\delta_k - \frac{M}{K}\right)^2 
% + \sum_{k=1}^{K-1} \lambda_k \left(\sum_{i=1}^{k} \delta_i - \pi_k\right)
% - \nu \left(\sum_{i=1}^{K} \delta_i - M\right),
% \end{equation}
where $\lambda_k$ ($k=1, \dots, K-1$) are the dual variables associated with the inequality constraints of the primal problem \eqref{eq:threshold_main3},
and $\nu$ is the dual variable associated with the equality constraint of the primal problem \eqref{eq:all_clients3}.

Based on the \ac{KKT} necessary and sufficient conditions for optimality \cite{boyd_book}, the optimal solution of our problem must satisfy the following:
\begin{enumerate}
    \item The primal equality constraint  \eqref{eq:all_clients3},
    \item The primal inequality constraints in \eqref{eq:threshold_main3}, which are $K-1$ inequalities in total,
    \item The dual constraints in \eqref{eq:threshold_dual}, which are also $K-1$ inequalities,
    \item The complementary slackness, given by
    \begin{equation}
    \label{eq:complementary_slackness}
    \lambda_k \left(\sum_{i=1}^{k} \delta_i - \pi_k\right) = 0, \quad \forall k\in \lbrace 1,2,\dots,K-1 \rbrace,
    \end{equation}
    which are $K-1$ equations related to the previous two conditions,
    \item Vanishing of the gradient of the Lagrangian with respect to the primal optimization variables,
    \begin{equation}
    \label{eq:vanishing_gradient}
    \frac{\partial \mathcal{L}}{\partial \delta_k} = 0, \quad \forall k\in \lbrace 1,2,\dots,K \rbrace,
    \end{equation}
    which gives $K$ equations in total.
\end{enumerate}
Furthermore, since the problem is convex, the solution that satisfies all \ac{KKT} conditions is unique.
To solve this system of equations and equalities, we will first perform some analytical manipulations based on these conditions. Then,
we will develop a hypothesis testing approach to get to the optimal solution in an efficient and quick way.
Based on the \ac{KKT} conditions, we can characterize the optimal solution as follows:
\begin{lemma}[Primal Variables and Lagrangian Multipliers]
\label{lem:Variables}

The optimal solution must satisfy
\begin{equation}
    \label{eq:subsequent}
    \delta_k = \delta_{k-1} + \frac{1}{2}\lambda_{k-1}, \quad \forall k=2, \dots, K.
\end{equation}
\end{lemma}
\begin{proof}[Proof:\nopunct]
From \eqref{eq:vanishing_gradient} we obtain
\begin{equation*}
    \frac{\partial \mathcal{L}}{\partial \delta_k} = 2 \left(\delta_k - \frac{M}{K}\right) + \sum_{i=k}^{K-1} \lambda_i - \nu = 0,
\end{equation*}
which can be re-arranged to obtain
\begin{equation}
    \label{eq:optimal_delta}
    \delta_k = \frac{M}{K} - \frac{1}{2} \sum_{i=k}^{K-1} \lambda_i + \frac{1}{2} \nu, \quad k=1, \dots, K.
\end{equation}
Based on \eqref{eq:optimal_delta}, and by evaluating $\delta_k - \delta_{k-1}$, we obtain \eqref{eq:subsequent}. 
\end{proof}

\begin{lemma}[Ordered Subsequent Clusters]
\label{lem:subsequent}
The optimal solution must satisfy
\begin{equation}
    \label{eq:optimality_check}
    \delta_k \geq \delta_{k-1},  \quad \forall k=2, \dots, K.
\end{equation}
\end{lemma}
\begin{proof}[Proof:\nopunct]
This is a direct consequence of Lemma~\ref{lem:Variables}  
since the Lagrangian multipliers $\lambda_k$ are non-negative based on the dual constraint \eqref{eq:threshold_dual}. 
\end{proof}

\begin{lemma}[Two Possibilities Per Complementary Slackness Condition]
\label{lem:two_possibilities}
The optimal solution must satisfy that either 
\begin{subequations}
    \label{eq:lemma2}
    \begin{alignat}{2}
        & \delta_k  = \delta_{k+1}, \quad  && \text{if}\quad  \lambda_k=0, \quad \forall k\in \lbrace 1,2,\dots,K-1 \rbrace, \label{eq:case1}\\
      & \text{or}&& \nonumber \\
       & \delta_k  = \pi_k - \sum_{i=1}^{k-1} \delta_i, \quad  && \text{if} \quad \lambda_k > 0, \quad \forall k\in \lbrace 1,2,\dots,K-1 \rbrace. \label{eq:case2} 
    \end{alignat}
\end{subequations}
% $\forall k\in \lbrace 1,2,\dots,K-1 \rbrace$.
\end{lemma}
\begin{proof}[Proof:\nopunct]
From the complementary slackness conditions in \eqref{eq:complementary_slackness}, the optimal solution must satisfy that either $\lambda_k=0$,
which results in \eqref{eq:case1} by substituting for $\lambda_k=0$ in \eqref{eq:subsequent}, or $\sum_{i=1}^{k} \delta_i - \pi_k = 0$, which can be re-arranged as in \eqref{eq:case2}.
These two possibilities correspond to satisfying the primal inequality constraints in \eqref{eq:threshold_main3} at strict inequality in the first case or strict equality in the second case.
The primal constraint in the latter case can be described as ``active''.
\end{proof}

%---------------------------------------------------------------------------------------------
\subsection{Hypothesis Testing Approach}
\label{sec:hypothesis}

The main bottleneck in solving \eqref{eq:dual_optimization} using the \ac{KKT} conditions is that we cannot know in advance which inequality constraints will be active and which ones will be inactive at the optimal solution. All possibilities are likely in the generic problem formulation in which the thresholds $\pi_k$ ($k=1,\dots,K-1$) can take any arbitrary ordered values. So, to overcome this bottleneck, we propose a hypothesis-testing approach to solve the \ac{KKT} conditions. 
We can start by guessing which inequality constraints in \eqref{eq:threshold_main3} will be active and which ones will be inactive. 
This will be a hypothesis that will be tested for validity.
Based on the initial assumption of the hypothesis, we can then obtain all primal optimization variables $\delta_k$ ($k = 1, 2, \dots, K$) based on \eqref{eq:case1} and \eqref{eq:case2} of Lemma~\ref{lem:two_possibilities}, in addition to \eqref{eq:all_clients3}. 
We will have $K$ independent equations and $K$ variables, so the solution of this system of equations will be unique.
After we obtain the solution, we can check for the validity of the hypothesis by examining two things. First, we check that there is no contradiction in the hypothesis's initial assumption by examining if all primal inequality constraints that were assumed to be inactive are indeed inactive after we substitute the obtained values of $\delta_k$ ($k = 1, 2, \dots, K$). 
Second, we check if the obtained $\delta_k$ ($k = 1, 2, \dots, K$) satisfy \eqref{eq:optimality_check}
of Lemma~\ref{lem:subsequent}. This check is for the optimality of the solution since the result in Lemma~\ref{lem:subsequent} is based on manipulations of the last \ac{KKT} condition about vanishing gradient of the Lagrangian. 
So, in summary, we need to make sure that the hypothesis is both (i) non-contradicting and (ii) generating the unique optimal solution of the problem. 

Since we have in total $K-1$ inequality constraints in \eqref{eq:threshold_main3}, we can have a maximum of $2^{K-1}$ hypotheses to be tested. In a practical deployment scenario,
the number of clusters $K$ is assumed to be limited to a single-digit number. So, checking all possible hypotheses would not be a prohibitively expensive step in such scenarios. 
Having said that, we also propose a search algorithm to find the correct hypothesis quickly. This can be particularly useful for scenarios of large $K$. 
However, before discussing the search algorithm, we will give numerical examples of hypothesis testing to clarify the concept. Then, we will formulate the hypothesis testing procedure in a systematic way that can be implemented in an algorithm.

\begin{example} [Three Scenarios of Hypothesis Testing]
\label{example:1}
Assume that $K=4$, $M=100$, $\pi_1=10$, $\pi_2=46$, and $\pi_3=80$. We can distinguish between three cases: 
\begin{itemize}
    \item Case A: Assuming that only the second inequality constraint is active is a contradicting hypothesis.
    \item Case B: Assuming that only the first and second inequality constraints are active is a non-contradicting but not an optimal hypothesis.
    \item Case C: Assuming that only the first inequality constraint is active is both a non-contradicting and optimal hypothesis. 
\end{itemize}
\end{example}

In Case A, since the second inequality constraint is assumed to be active, we should use \eqref{eq:case2} for $k=2$, and use \eqref{eq:case1} for $k=1$ and $k=3$. The system of four equations that should be solved is given as:
\begin{subequations}
\label{eq:Case_A}
\begin{alignat}{2}
    \delta_1 & =  \delta_2, \label{eq:k_1}\\
    \delta_2 & =  \pi_2 - \delta_1, \label{eq:k_2}\\
    \delta_3 & =  \delta_4, \label{eq:k_3}\\
   \delta_4 & =  M - (\delta_1+\delta_2+\delta_3). \label{eq:k_4}
\end{alignat}
\end{subequations}
The numerical solution of these equations yields $\delta_1=\delta_2=23$, and $\delta_3=\delta_4=27$. By checking the first inequality constraint, we will find a contradiction since 
$\delta_1=23 > \pi_1=10$. So, this is a contradicting hypothesis.

In Case B, we will have the same system of equations \eqref{eq:Case_A}, but with replacing \eqref{eq:k_1} by $\delta_1=\pi_1$. The solution of the system of four equations in this case yields $\delta_1=10$, $\delta_2=36$, and $\delta_3=\delta_4=27$. By checking for all inequality constraints, we do not find any contradiction. However, this solution is not optimal since $\delta_2=36 > \delta_3 = 27$, and this violates the optimality condition \eqref{eq:optimality_check} that was proven in Lemma~\ref{lem:subsequent}.

Finally, in Case C, we will have \eqref{eq:k_1} replaced by $\delta_1=\pi_1$, and \eqref{eq:k_2} replaced by $\delta_2=\delta_3$. The other two equations will be the same.
The solution in this case yields $\delta_1=10$, and $\delta_2=\delta_3=\delta_4=30$. This solution does not violate any inequality constraint and satisfies the optimality check in  \eqref{eq:optimality_check}. So, it is the optimal solution of the problem.

To formulate the hypothesis testing approach in a systematic algorithm, we use the notation $\mathbf{h}$ for a binary vector of $K$ entries that indicates which primal inequality constraints are assumed to be active by the hypothesis. The final entry in $\mathbf{h}$ is for the equality constraint \eqref{eq:all_clients3}, which is by definition always active. 
The entry of $\mathbf{h}$ that corresponds to the index $k$ of an active constraint is given the value $1$, and the entry that corresponds to the index of an inactive inequality is given the value $0$.
So, for Case A in the previous example we have $\mathbf{h}=(0,1,0,1)$. For Case B, we have $\mathbf{h}=(1,1,0,1)$, and for Case C, we have $\mathbf{h}=(1,0,0,1)$.

We also use the notation $\mathbf{p}(\mathbf{h})$ for a vector that is obtained from $\mathbf{h}$ by indicating the indices of the entries that have value $1$.
So, in our previous examples, if $\mathbf{h}=(0,1,0,1)$, then $\mathbf{p}=(2,4)$, and if $\mathbf{h}=(1,1,0,1)$, then $\mathbf{p}=(1,2,4)$, and if $\mathbf{h}=(1,0,0,1)$, then $\mathbf{p}=(1,4)$, and so on.
It should be noted that while the number of entries of $\mathbf{h}$ is fixed at $K$, the number of entries in $\mathbf{p}(\mathbf{h})$ varies depending on the sum of all entries of
the hypothesis $\mathbf{h}$. Furthermore, we use the notation $\mathbf{p}^{(i)}$ to indicate the i-th entry of $\mathbf{p}$. So, for example, if $\mathbf{p}=(1,2,4)$, then $\mathbf{p}^{(1)}=1$, $\mathbf{p}^{(2)}=2$ and $\mathbf{p}^{(3)}=4$.

Based on the introduced notations for $\mathbf{h}$ and $\mathbf{p}$, we can obtain that for a given hypothesis $\mathbf{h}$, the resulting system of equations
to obtain $\delta_k$, $k=1,\dots,K$, will have the following generalized formula to obtain the solution for any $\mathbf{p}(\mathbf{h})$:
\begin{equation}
\label{eq:solution_delta}
\delta_i = \frac{\pi_{\mathbf{p}^{(j)}}-\pi_{\mathbf{p}^{(j-1)}}}{\mathbf{p}^{(j)}-\mathbf{p}^{(j-1)}}, \quad \forall i \in \left\{\mathbf{p}^{(j-1)}+1, \dots, \mathbf{p}^{(j)} \right\},
\end{equation}   
where $j=1,2,\dots,|\mathbf{p}|$, $\mathbf{p}^{(0)}=0$, $\pi_0=0$ and $\pi_K=M$.
Implementing \eqref{eq:solution_delta} in an algorithm is straightforward by running a for-loop for the values of $j$, and obtaining all $\delta_i$'s that are associated with $j$ in \eqref{eq:solution_delta}.
After running through all $j$ values, all $\delta_i$, $i=1,\dots,K$ will be obtained.

%---------------------------------------------------------------------------------------------
\subsection{Quick Search Over Space Of Possible Hypotheses}
\label{sec:search}

To be able to find the non-contradicting and optimal hypothesis in the minimum number of steps, we need to make some theoretical characterization
of the space of possible hypotheses so that we can exclude many possibilities directly and find the optimal solution in a quick search.
 
\begin{theorem}[Inactive Primal Constraints]
\label{theorem:inactive_Constraints} 
If $\pi_k$ in the $k$-th inequality constraint in \eqref{eq:threshold_main3} satisfies
\begin{equation}
\label{eq:theorem1}
\pi_k > \frac{k}{K}M, \quad  k\in \lbrace 1, 2, \dots, K-1 \rbrace,
\end{equation}
then this constraint will be inactive at the optimal solution of the problem.
\end{theorem}

\begin{proof}[Proof:\nopunct]
Assume that \eqref{eq:theorem1} is true and the $k$-th inequality constraint is active. The latter means
\begin{equation}
    \label{eq:theorem1_1}
    \sum_{i=1}^{k} \delta_i = \pi_k.
\end{equation}
If the solution is optimal, then we must have \eqref{eq:optimality_check} satisfied as proven in  Lemma~\ref{lem:subsequent}.
Thus, we must have $\delta_k \geq \delta_{k-1} \geq \dots \geq \delta_1$. Consequently, we must have that $\delta_k \geq \frac{\pi_k}{k}$,
which can be bounded based on \eqref{eq:theorem1} to have
\begin{equation}
    \label{eq:theorem1_2}
    \delta_k \geq \frac{\pi_k}{k} > \frac{M}{K}. 
\end{equation}
Furthermore, by simple manipulations of the equality constraint in \eqref{eq:all_clients3} with \eqref{eq:theorem1_1}, we obtain 
\begin{equation}
    \label{eq:theorem1_3}
    \sum_{i=k+1}^{K} \delta_i = M - \pi_k.
\end{equation}
If the solution is optimal, then we must have \eqref{eq:optimality_check} satisfied, which yields
$\delta_K \geq \delta_{K-1} \geq \dots \geq \delta_{k+1}$. Consequently, we must have $\delta_{k+1} \leq \frac{M-\pi_k}{K-k}$,
which can be bounded based on \eqref{eq:theorem1} to have
\begin{equation}
    \label{eq:theorem1_4}
    \delta_{k+1} \leq \frac{M-\pi_k}{K-k} < \frac{M-\frac{k}{K}M}{K-k}=\frac{M\left(\frac{K-k}{K}\right)}{K-k}=\frac{M}{K}. 
\end{equation}
From \eqref{eq:theorem1_2} and \eqref{eq:theorem1_4}, we find that $\delta_k > \frac{M}{K} > \delta_{k+1}$. However,
based on the optimality check in \eqref{eq:optimality_check}, we must have $\delta_{k+1} \geq \delta_k$. So, there is a contradiction.
Thus, if \eqref{eq:theorem1} is true, then the associated constraint must be inactive at optimality.
\end{proof}

The statement in Theorem~\ref{theorem:inactive_Constraints} can help in excluding many hypotheses directly from the search algorithm. Furthermore,
we can narrow the space of possible hypotheses even further based on the following theorem.

\begin{theorem}[Generalized Check For Inactive Primal Constraints]
\label{theorem:Generalized} 
Assume that we have any arbitrary two integers $a$ and $b$ such that $1 \leq a <b \leq K$.
If $\pi_k$ associated with the $k$-th inequality constraint in \eqref{eq:threshold_main3} satisfies
\begin{equation}
\label{eq:theorem2}
\frac{\pi_k-\pi_a}{k-a} > \frac{\pi_b-\pi_a}{b-a}, \quad  k\in \lbrace a+1, a+2, \dots, b-1 \rbrace,
\end{equation}
then this constraint will be inactive at the optimal solution of the problem.
\end{theorem}

Before presenting the proof, it should be noted that the statement in Theorem~\ref{theorem:Generalized} is a generalized
version of the statement in Theorem~\ref{theorem:inactive_Constraints}. The latter is a special case in which we have $a=0$ and $b=K$, and similar to 
\eqref{eq:solution_delta}, we use the notation $\pi_0=0$ and $\pi_K=M$.

\begin{proof}[Proof:\nopunct]
Assume that both \eqref{eq:theorem2} is true and the $k$-th inequality constraint is active. The latter means that \eqref{eq:theorem1_1} and \eqref{eq:theorem1_3} are true
as discussed in the proof of Theorem~\ref{theorem:inactive_Constraints}.
Furthermore, we need to use the $a$-th and the $b$-th primal constraints in the construction of the proof. So, we have
\begin{subequations}
\label{eq:theorem2_1}
\begin{alignat}{2}
    \sum_{i=1}^{a} \delta_i & \leq  \pi_a, \label{eq:theorem2_1a}\\
    \sum_{i=1}^{b} \delta_i & \leq  \pi_b. \label{eq:theorem2_1b}
\end{alignat}
\end{subequations}
By simple manipulations of \eqref{eq:theorem1_1} and \eqref{eq:theorem2_1a}, we obtain that $\sum_{i=a+1}^{k} \delta_i \geq \pi_k - \pi_a$.  
Furthermore, if the solution is optimal, then we must have \eqref{eq:optimality_check} satisfied.
Thus, we must have $\delta_k \geq \delta_{k-1} \geq \dots \geq \delta_{a+1}$. Consequently, we must have that $\delta_k \geq \frac{\pi_k - \pi_a}{k-a}$,
which can be bounded based on \eqref{eq:theorem2} to have
\begin{equation}
\label{eq:theorem2_2}
\delta_k \geq \frac{\pi_k-\pi_a}{k-a} > \frac{\pi_b-\pi_a}{b-a}. 
\end{equation}
Next, we can do simple manipulations on the equality constraint in \eqref{eq:all_clients3} with \eqref{eq:theorem2_1b} to obtain that $\sum_{i=b+1}^{K} \delta_i \geq M - \pi_b$,
which can be combined with \eqref{eq:theorem1_3} to obtain
$\sum_{i=k+1}^{b} \delta_i \leq (M - \pi_k) - (M - \pi_b) = \pi_b-\pi_k$.
Moreover, if the solution is optimal, then we must have \eqref{eq:optimality_check} satisfied, which yields
$\delta_{b} \geq \delta_{b-1} \geq \dots \geq \delta_{k+1}$. Consequently, we must have that 
\begin{equation}
\label{eq:theorem2_3}
\delta_{k+1} \leq \frac{\pi_b-\pi_k}{b-k}.
\end{equation}
Looking back at \eqref{eq:theorem2}, we can re-write the inequality as 
$\pi_k > \pi_a + \frac{k-a}{b-a}(\pi_b-\pi_a)$,
which can be used to bound \eqref{eq:theorem2_3} as follows
\begin{alignat}{2}
    \delta_{k+1} & \leq  \frac{\pi_b-\pi_k}{b-k}  \nonumber \\
                   & < \frac{\pi_b-\pi_a - \frac{k-a}{b-a}(\pi_b-\pi_a)}{b-k} \nonumber \\
   & = \frac{(\pi_b-\pi_a)\left(1-\frac{k-a}{b-a}\right)}{b-k} \nonumber \\
  & = \frac{\pi_b - \pi_a}{b-a} \label{eq:theorem2_4}
\end{alignat}
From \eqref{eq:theorem2_2} and \eqref{eq:theorem2_4}, we find that $\delta_k > \frac{\pi_b - \pi_a}{b-a} > \delta_{k+1}$. However,
based on the optimality check in \eqref{eq:optimality_check}, we must have $\delta_{k+1} \geq \delta_k$. So, there is a contradiction.
Thus, if \eqref{eq:theorem2} is true, then the associated constraint must be inactive at optimality.
\end{proof}

\begin{figure}
    \centering
    \includegraphics[width=\linewidth]{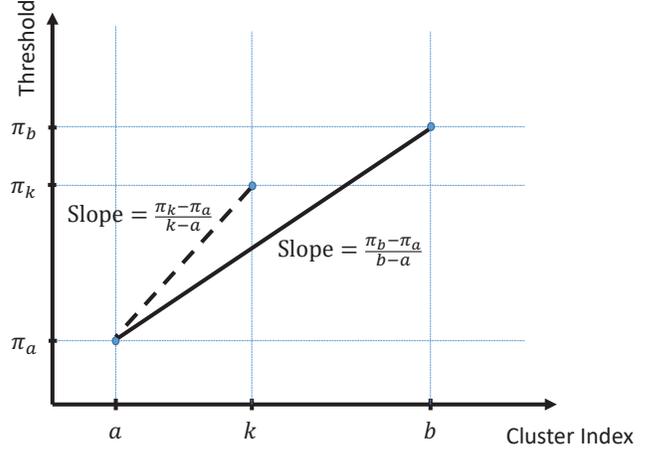}
    \caption{Geometric Interpretation of \eqref{eq:theorem2} in Theorem~\ref{theorem:Generalized}.}
    \label{fig:Theorem2}
\end{figure}

Actually, we can adopt a geometric interpretation of Theorem~\ref{theorem:Generalized} as shown in Fig.~\ref{fig:Theorem2},
since the left term of the inequality in \eqref{eq:theorem2} is the slope of the line connecting the two points $(a,\pi_a)$ and $(k,\pi_k)$,
while the right term is the slope of the line connecting the two points $(a,\pi_a)$ and $(b,\pi_b)$. Thus, we can equivalently reinterpret Theorem~\ref{theorem:Generalized} 
as stating that the $k$-th inequality constraint will be inactive at the optimal solution if the point $(k,\pi_k)$ is located above the line segment that connects
the two points $(a,\pi_a)$ and $(b,\pi_b)$. This statement is true for any value of $a<k$ and any value of $b>k$.  
Hence, we can establish a nice geometric interpretation of our problem. Actually, finding the active inequality constraints can be interpreted as
finding the lower convex hull in a two-dimensional Cartesian coordinate system for the set of inequality thresholds, represented as:
\begin{equation}
    \label{eq:set_of_thresholds}
    \mathcal{P} = \lbrace (k, \pi_k) : k=0,1,\dots,K \rbrace.
\end{equation}
 
Furthermore, the active constraints are only those that are located on the vertices of the piecewise linear boundary of the lower convex hull of $\mathcal{P}$,
as illustrated in Fig.~\ref{fig:Example1}, which shows a numerical example of $\mathcal{P}=\lbrace (0,0), (1,10), (2,46), (3,80), (4,100) \rbrace$.
This particular $\mathcal{P}$ is related to the previously discussed Example~\ref{example:1}. As shown in Fig.~\ref{fig:Example1},
the point $(1,10)$ is a vertex in the boundary of the lower convex hull of $\mathcal{P}$, while the points $(2,46)$ and $(3,80)$ are not located on the boundary
of the lower convex hull. Thus, the first constraint should be active at the optimal solution, while the second and third constraints should be inactive.
This conclusion was already verified in Example~\ref{example:1}. So, the geometric interpretation is consistent with the analytical algebraic derivations. 

\begin{figure}
    \centering
    \includegraphics[width=.8\linewidth]{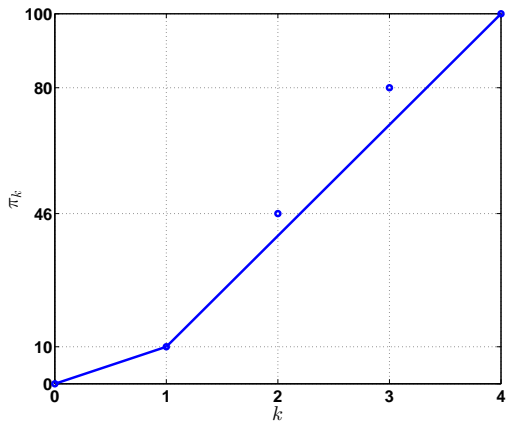}
    \caption{The lower convex hull for Example~\ref{example:1}.}
    \label{fig:Example1}
\end{figure}

The significance of the geometric interpretation is that it enables us to define an efficient and quick procedure to find the optimal solution by
searching for the boundary of the lower convex hull of $\mathcal{P}$ of our problem. This can be done sequentially segment-by-segment 
starting from the lower-left corner point $(0,0)$ until we reach the upper right corner point $(K,M)$, as summarized in Algorithm~\ref{alg:optimization}.

\begin{algorithm}[htbp]
\caption{Optimal Solution of Problem \eqref{eq:main_optimization3}}
\label{alg:optimization}
\begin{algorithmic}[1]
\renewcommand{\algorithmicrequire}{\textbf{Input:}}
\renewcommand{\algorithmicensure}{\textbf{Output:}}
\REQUIRE $\pi_k, \forall k \in  \lbrace 0, 1, \dots, K \rbrace$, where $\pi_0=0$ and $\pi_K=M$.
\ENSURE  The optimal solution of \eqref{eq:main_optimization3}: $\delta_k, \forall k \in  \lbrace 1, \dots, K \rbrace$
%\\ \textit{Initialisation} :
\STATE Initialize $a=0$, $\pi_a=0$.
%\\ \textit{LOOP Process}
\WHILE {$a \neq K$}
    \STATE Find $b^*$ that solves the problem 
    \begin{equation*}
        \min_{b\in \lbrace a+1, \dots, K \rbrace} \frac{\pi_b-\pi_a}{b-a}. 
    \end{equation*}
    \STATE Use the obtained $b^*$ to set the output values 
    \begin{equation*}
        \delta_k = \frac{\pi_{b^*}-\pi_a}{b^*-a}, \quad \forall k \in \lbrace a+1, \dots, b^* \rbrace.
    \end{equation*}
    \STATE Update $a=b^*$, $\pi_a=\pi_{b^*}$.
    %  \IF {($i \ne 0$)}
    % \STATE statement..
    % \ENDIF
\ENDWHILE
%\RETURN $P$ 
\end{algorithmic} 
\end{algorithm}

%---------------------------------------------------------------------------------------------
\section{Performance Analysis and Comparison}
\label{sec:simulations}

In this section, as a numerical example, we compare the performance of a conventional federated learning scheme (in which $N$ clients are randomly selected per iteration $t$ regardless of their number of data samples) with our suggested \ac{PFL} scheme (which follows \ALG{alg:csgd} and the solution in \ALG{alg:optimization}). We analyze the proposed \ac{PFL} method using different datasets such as the MNIST \cite{lecun2010mnist}, the federated EMNIST \cite{cohen2017emnist}, and the federated Shakespeare \cite{main_FL,shakespeare2007complete} datasets and neural network architectures such as \ac{MLP}, \ac{CNN}, and \ac{RNN} to cover a wide range of federated learning applications from the literature. We used Python programming language and Tensorflow Federated framework \cite{tensorflow_federated} for the simulations.

Without loss of generality, we assume a scenario in which all clients have similar processors' speeds, and the variation in computation time $\tau_m$ is solely dependent on the number of data sample per client, i.e., $\tau_m = \sigma n_m$,
% \begin{equation}
%     \tau_m = \sigma n_m,
% \end{equation}
where $\sigma$ is the computation time per data sample.  

\subsection{MNIST Dataset}

First, we have validated the proposed method using the MNIST dataset \cite{lecun2010mnist}, which consists of hand-written digits. In the simulation scenario, we assumed that there are $M=1500$ clients. We assumed that the number of data samples (i.e., images of hand-written digits) per client $n_m$ is uniformly distributed between $n_{\min}=10$ samples and $n_{\max}=70$ samples. We compared the results of \ac{PFL} to the conventional federated learning scheme where there is only one cluster. We considered four possible scenarios for the ratio between the variation in computation time to communication time in \eqref{eq:choosing_K}. We performed the simulations for $1$, $2$, $3$, and $4$ clusters, respectively. We also studied the effect of the number of sub-channels on the performance of \ac{PFL}. Hence, we varied the number of sub-channels as $1$, $2$, $4$, and $8$, respectively. Each sub-channel can be used by a single client only at a given time instant.
We used two different deep learning models for the training of the MNIST dataset. The first model is a two-layer \ac{CNN} model with $32$ and $64$ $5\times 5$ filters, respectively.
The second model is an \ac{MLP} model with two layers with $200$ nodes in each layer.
We run the training algorithm for up to $1000$ rounds of global model (server) updates. We obtained the model prediction accuracy by testing the trained model on a separate test dataset as a function of the number of communication rounds.
The batch size is chosen as $16$, and the number of local epochs is set to $1$. The client learning rate is optimized using a multiplicative parameter grid, and the server learning rate (the scaling factor) is set to $1$.
The number of required rounds to reach $97\%$ accuracy for the \ac{CNN} model trained on the MNIST dataset is shown in \TAB{tab:mnist-cnn}. Similarly, the number of required rounds to reach $96\%$ accuracy for the \ac{MLP} model trained on the MNIST dataset is shown in \TAB{tab:mnist-mlp}. 
% The evolution of the test accuracy with respect to the number of communication rounds is provided in \FIG{fig:mnist-cnn} and \FIG{fig:mnist-mlp} for training with the \ac{CNN} and \ac{MLP} models, respectively. The test accuracy values in both figures represent the maximum test accuracy obtained until that round.

\begin{table}[htbp]
    \centering
    \caption{Number of required rounds to reach $97\%$ accuracy for the CNN model on the MNIST dataset. Columns represent the number of clusters and rows represent number of sub-channels}
    \label{tab:mnist-cnn}
    \begin{tabular}{c|c|c|c|c}
        K\textbackslash N & 1 & 2 & 4 & 8 \\
        \hline
        1 & 273 (0\%)  & 180 (0\%)  & 152 (0\%)  & 118 (0\%) \\
        2 & 186 (32\%) & 129 (28\%) & 109 (28\%) & 106 (10\%) \\
        3 & 139 (49\%) & 129 (28\%) & 114 (25\%) & 100 (15\%) \\
        4 & 131 (52\%) & 122 (32\%) & 105 (30\%) & 97 (18\%) \\
    \end{tabular}
\end{table}

% \begin{figure}[htbp]
% 	\centering
% 	\includegraphics[width=\figscale\linewidth]{fig/mnist_cnn.eps}
% 	\caption{Accuracy vs. communication round for MNIST dataset trained with CNN model.}
% 	\label{fig:mnist-cnn}
% \end{figure}

\begin{table}[htbp]
    \centering
    \caption{Number of required rounds to reach $96\%$ accuracy for the MLP model on the MNIST dataset. Columns represent the number of clusters and rows represent number of sub-channels}
    \label{tab:mnist-mlp}
    \begin{tabular}{c|c|c|c|c}
        K\textbackslash N & 1 & 2 & 4 & 8 \\
        \hline
        1 & 981 (0\%)  & 767 (0\%)  & 528 (0\%)  & 445 (0\%) \\
        2 & 754 (23\%) & 514 (32\%) & 480 (9\%) & 436 (2\%) \\
        3 & 613 (38\%) & 451 (41\%) & 479 (9\%) & 432 (3\%) \\
        4 & 506 (48\%) & 472 (38\%) & 443 (16\%) & 424 (5\%) \\
    \end{tabular}
\end{table}

% \begin{figure}[htbp]
% 	\centering
% 	\includegraphics[width=\figscale\linewidth]{fig/mnist_mlp.eps}
% 	\caption{Accuracy vs. communication round for MNIST dataset trained with MLP model.}
% 	\label{fig:mnist-mlp}
% \end{figure}

As seen in \TAB{tab:mnist-cnn} and \TAB{tab:mnist-mlp}, applying clustered scheduling enables achieving faster convergence (in terms of the number of iterations) in comparison with a baseline federated learning approach with no clustering.
The conventional federated learning scheme with the \ac{CNN} model achieves $97\%$ accuracy in $273$, $180$, $152$, and $118$ rounds with $1$, $2$, $4$, and $8$ sub-channels, respectively. The proposed \ac{PFL} scheme provides $32\%$ gain compared to the conventional federated learning scheme in terms of required rounds by employing only $2$ clusters. The gain provided by clustering is improved for each additional cluster in the system. The marginal gain provided by the \ac{PFL} scheme decreases with the increasing number of sub-channels since the effective number of data samples per communication round is already enough for the conventional federated learning.

The benefit of the proposed clustered pipelining scheme can be realized in a different way by decreasing the bandwidth requirements. Thus, we can decrease the number of sub-channels needed to achieve the desired accuracy value. For instance, by using $4$ clusters with $1$ sub-channel, we reach the same performance as the one achieved by the conventional federated learning with $4$ sub-channels.

The \ac{MLP} model converges slower compared to the \ac{CNN} model. The conventional federated learning scheme achieves $96\%$ accuracy in $981$, $767$, $528$, and $445$ rounds with $1$, $2$, $4$, and $8$ sub-channels, respectively. Similar to the results with the \ac{CNN} model, introducing clustering into the federated learning setup decreases the number of required training rounds significantly. However, the \ac{PFL} method provides little to no gain compared to the conventional federated learning scheme if the number of sub-channels is high, e.g., if the number of sub-channels is $8$ in our simulation setup.

\subsection{Federated EMNIST Dataset}

In this section, we use the extended version of the MNIST dataset, namely, EMNIST \cite{cohen2017emnist}. We use the federated EMNIST dataset in which the data is already distributed to the clients by their writing styles. The EMNIST dataset has letters additionally; however, we will use only the digits dataset. The federated EMNIST digits dataset has a total of $341,873$ train and $40,832$ test samples distributed among $3,383$ clients. There are $10$ classes in the dataset, each corresponding to a hand-written digit. The data distribution among the clients is shown in \FIG{fig:emnist-histogram}. As seen from \FIG{fig:emnist-histogram}, the distribution of the number of data samples among the clients is unbalanced and not uniform.
We formulate the client clustering problem as shown in \eqref{eq:main_optimization3} for clustering the clients for the federated EMNIST dataset, and then we solve the problem efficiently using \ALG{alg:optimization}. Accordingly, in the training step, we first cluster the clients using the proposed optimization problem. Then, we train the model using the \ac{PFL} scheme. We performed the simulations for $1$, $2$, $3$, and $4$ clusters and $1$, $2$, $4$, and $8$ sub-channels, respectively. We used a two-layer \ac{CNN} model with $32$ and $64$ $5\times 5$ filters as the neural network model. The batch size is chosen as $16$, and the learning rate is optimized using a multiplicative parameter grid. The number of local epochs is chosen as $1$ in order to assess the convergence speed. The required number of rounds to achieve $97\%$ accuracy for the EMNIST digits dataset is given in \TAB{tab:emnist}.

\begin{figure}[htbp]
    \centering
    \includegraphics[width=\figscale\linewidth]{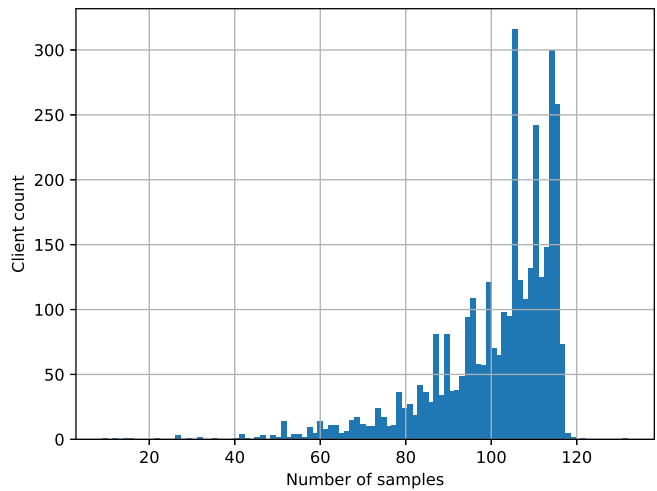}
    \caption{Data distribution in the federated EMNIST dataset.}
    \label{fig:emnist-histogram}
\end{figure}

\begin{table}[htbp]
    \centering
    \caption{Number of required rounds to reach $97\%$ accuracy for the CNN model on EMNIST dataset. Columns represent the number of clusters and rows represent number of sub-channels}
    \label{tab:emnist}
    \begin{tabular}{c|c|c|c|c}
        K\textbackslash N & 1 & 2 & 4 & 8 \\
        \hline
        1 & 686 (0\%)  & 221 (0\%)  & 189 (0\%)  & 170 (0\%) \\
        2 & 274 (60\%) & 184 (17\%) & 178 (6\%)  & 160 (6\%) \\
        3 & 191 (72\%) & 182 (18\%) & 158 (16\%) & 158 (7\%) \\
        4 & 191 (72\%) & 175 (21\%) & 156 (17\%) & 160 (6\%) \\
    \end{tabular}
\end{table}

The number of communication rounds required is $686$ for the $1$ sub-channel and $1$ cluster case. The number of communication rounds drops by $64\%$ to $274$ rounds by introducing $1$ additional cluster. The number of communication rounds gets lower and lower with the increasing number of clusters. The gain from the proposed \ac{PFL} scheme decreases if the number of available sub-channels increases. Therefore, we could conclude that the proposed scheme may not be feasible if the number of sub-channels is sufficiently large, e.g., more than $4$ sub-channels in this case.
% The test accuracy vs. the number of global communication rounds for the federated EMNIST dataset is provided in \FIG{fig:emnist}. The test accuracy values in \FIG{fig:emnist} represent the maximum test accuracy obtained until that round. Using the proposed pipelining scheme in the training of the federated EMNIST dataset significantly improves the convergence speed of the model. 
As seen from \TAB{tab:emnist}, the model converges faster for a higher number of clusters in the system. However, the marginal gain obtained by each additional cluster decreases with the number of clusters already present in the system.

% \begin{figure}[htbp]
% 	\centering
% 	\includegraphics[width=\figscale\linewidth]{fig/emnist.eps}
% 	\caption{Accuracy vs. communication round for the federated EMNIST dataset trained with CNN model.}
% 	\label{fig:emnist}
% \end{figure}

\subsection{Federated Shakespeare Dataset}
For modeling time-series data, we used federated Shakespeare dataset \cite{main_FL,shakespeare2007complete} for training a stacked two-layer \ac{LSTM} network. Federated Shakespeare dataset is built by combining all works of Shakespeare.
The dataset consists of the speaking roles of $715$ characters in each play, corresponding to the clients in our federated learning setup. The dataset is unbalanced since most of the plays have a few lines, and a small number of clients have a high count of lines in their plays. Additionally, the number of data samples is not distributed uniformly to the clients. It is also clear that the dataset is not \ac{iid} since the dataset is distributed to the clients according to the played character.
% The data distribution among the clients is shown in \FIG{fig:shakespeare-histogram}.
We use \ALG{alg:optimization} for determining the clusters for the clients. Then, we train the model using the proposed \ac{PFL} scheme. 
The neural network model used for this task has an embedding layer at the input, which embeds the input characters into an $8$ dimensional space. After the embedding layer, it has two \ac{LSTM} layers with $256$ hidden nodes. Finally, it has an output layer with softmax activation with an output for each character.
In training, we set the batch size to be $10$ and optimized the learning rate through a multiplicative parameter grid.
The simulation results for the application of \ac{PFL} for the training of the two-layer \ac{LSTM} network on the Shakespeare dataset are given in \TAB{tab:shakespeare}. The rows represent the number of clusters, and the columns represent the number of sub-channels. The values are given in terms of the number of rounds required to reach $40\%$ accuracy in the training. 
By examining the simulation results for both federated the EMNIST and the Shakespeare datasets, we conclude that the proposed pipelining scheme improves the results for non-uniformly distributed data more significantly.

% \begin{figure}[htbp]
%     \centering
%     \includegraphics[width=\figscale\linewidth]{fig/shakespeare_data_histogram.eps}
%     \caption{Data distribution in the federated Shakespeare dataset.}
%     \label{fig:shakespeare-histogram}
% \end{figure}

\begin{table}[htbp]
    \centering
    \caption{Number of required rounds to reach $40\%$ accuracy for the two-layer LSTM model on Shakespeare dataset. Columns represent the number of clusters and rows represent number of sub-channels}
    \label{tab:shakespeare}
    \begin{tabular}{c|c|c|c|c}
        K\textbackslash N & 1 & 2 & 4 & 8 \\
        \hline
        1 & 1182 (0\%) & 661 (0\%)  & 429 (0\%)  & 354 (0\%) \\
        2 & 605 (49\%) & 379 (43\%) & 273 (36\%) & 213 (40\%) \\
        3 & 419 (65\%) & 282 (57\%) & 239 (44\%) & 198 (44\%) \\
        4 & 372 (69\%) & 257 (61\%) & 215 (50\%) & 190 (46\%) \\
    \end{tabular}
\end{table}

% \subsection{Training Stability}

% The proposed \ac{PFL} scheme also improves the training stability for the models. The \ac{PFL} scheme increases the effective number of data samples per communication round by clustering the clients. By using more data for the gradient estimation, the actual value of the gradients could be estimated more accurately in each round. Therefore, the stability of the training is improved by each additional cluster in the system. In order to show the training stability, we plotted the test accuracy with respect to the number of communication rounds without taking the maximum value up to that round. In this way, the ripples in the test accuracy are more apparent. As an example, the raw test accuracy values for the federated EMNIST dataset trained with a \ac{CNN} is shown in \FIG{fig:stability-emnist}. As seen from \FIG{fig:stability-emnist}, the amount of ripples decreases with the larger number of clusters. Thus, we achieve a more stable model by employing the proposed \ac{PFL} scheme.

% \begin{figure}[htbp]
% 	\centering
% 	\includegraphics[width=\figscale\linewidth]{fig/stability_emnist.eps}
% 	\caption{The evolution of the test accuracy with respect to communication rounds for the federated EMNIST dataset.}
% 	\label{fig:stability-emnist}
% \end{figure}

%---------------------------------------------------------------------------------------------
\section{Conclusions}
\label{sec:conclusions}

One of the main challenges of federated learning is the heterogeneity across the participating clients in terms of their computation duration and the size of the local data set. 
To mitigate this challenge, we have proposed using communication pipelining to enhance the spectrum utilization efficiency and speed up the convergence of federated learning. The main idea is to cluster clients based on their computation time and then schedule a balanced mixture of clients from the different clusters.
The faster clients access the spectrum first, while the slower clients are still running their computation.
By allowing more clients to be involved in each iteration, the time required to converge into a targeted prediction accuracy of the trained model can be reduced significantly.
The time per iteration does not change in comparison with conventional federated learning. Nevertheless, the required number of iterations and the total time to train the model are reduced.
We provided a generic formulation for the optimal client clustering under different settings and an efficient algorithm for obtaining the optimal solution.
We demonstrated that the required number of communication rounds also depends on the number of sub-channels, and the gain obtained from the proposed method is lower for a high number of sub-channels.
% The proposed method also provides better training stability due to the involvement of more data in each round.

\bibliography{main}
\bibliographystyle{IEEEtran}

\end{document}